\documentclass[11pt]{article}

\usepackage[a4paper,top=2cm,bottom=2cm,left=3cm,right=3cm,marginparwidth=1.75cm]{geometry}

\usepackage{cite}
\usepackage{xcolor}
\usepackage[colorlinks=true, allcolors=blue]{hyperref}
\usepackage{amsmath,amssymb,amsfonts,amsthm, nccmath, mathrsfs}
\usepackage{graphicx}
\usepackage{textcomp}

\usepackage{booktabs}
\usepackage{subcaption}

\DeclareMathOperator*{\argmin}{argmin}
\usepackage[ruled,vlined]{algorithm2e}
\SetKwInput{KwIn}{Input}
\SetKwInput{KwOut}{Output}
\SetKwInput{Parameters}{Parameters}
\SetKwInput{Initialize}{Initialization}
\SetKwRepeat{Do}{Repeat}{Until}

\usepackage{enumitem}

\usepackage{bm}
\newtheorem{prop}{Proposition}

\usepackage{xspace}
\newcommand{\SECM}{\textsf{Soft-ECM}\xspace}

\usepackage{booktabs}
\makeatletter
\newcommand{\dashrule}[1][black]{%
  \color{#1}\rule[\dimexpr.5ex-.2pt]{4pt}{.4pt}\xleaders\hbox{\rule{4pt}{0pt}\rule[\dimexpr.5ex-.2pt]{4pt}{.4pt}}\hfill\kern0pt%
}
\newcommand{\rulecolor}[1]{ %
  \def\CT@arc@{\color{#1}}%
}
\makeatother
\newcommand*\widebar[1]{%
  \hbox{%
    \vbox{%
      \hrule height 0.5pt 
      \kern0.4ex
      \hbox{%
        \kern-0.0em
        \ensuremath{#1}%
        \kern-0.0em
      }%
    }%
  }%
} 

\begin{document}

\title{\SECM: An extension of Evidential C-Means\\
for complex data\footnote{This article has been published in the proceedings of the IEEE conference on Fuzzy Systems (Fuzz), 2025}}

\author{\textbf{Armel Soubeiga}, \textit{Clermont Auvergne INP, ENSM} \\
\textit{St Etienne, UMR 6158 CNRS, LIMOS,}  \\
\textit{armel.soubeiga@uca.fr}\\
\textbf{Thomas Guyet}, \textit{AIStroSight/Inria, UCBL, HCL}\\
\textit{thomas.guyet@inria.fr}\\
\textbf{Violaine Antoine}, \textit{Clermont Auvergne INP, ENSM}\\
\textit{St Etienne, UMR 6158 CNRS, LIMOS}  \\
\textit{violaine.antoine@uca.fr}
}
\date{}
\maketitle

\begin{abstract}
Clustering based on belief functions has been gaining increasing attention in the machine learning community due to its ability to effectively represent uncertainty and/or imprecision. However, none of the existing algorithms can be applied to complex data, such as mixed data (numerical and categorical) or non-tabular data like time series. Indeed, these types of data are, in general, not represented in a Euclidean space and the aforementioned algorithms make use of the properties of such spaces, in particular for the construction of barycenters. 
In this paper, we reformulate the \textit{Evidential C-Means} (ECM) problem for clustering complex data. We propose a new algorithm, \SECM, which consistently positions the centroids of imprecise clusters requiring only a semi-metric. Our experiments show that \SECM present results comparable to conventional fuzzy clustering approaches on numerical data, and we demonstrate its ability to handle mixed data and its benefits when combining fuzzy clustering with semi-metrics such as DTW for time series data.

\noindent\underline{keywords:} Evidential clustering, credal partition, time series data, optimization
\end{abstract}

\section{Introduction}
In numerous real-world applications, only partial information is available for objects, and in these cases hard clustering can result in poor accuracy. To address this problem and capture the degree of ambiguity regarding the class membership of the objects, soft clustering variants have been proposed, including Fuzzy C-Means\cite{softecm13} and Evidential C-Means\cite{softecm00}. These variants allow to describe the uncertainty and/or imprecision in the partition. 
Fuzzy C-Means (FCM)~\cite{softecm13} is based on fuzzy set theory, while Evidential C-Means (ECM) clustering is based on belief function theory~\cite{softecm00}. ECM allows objects to belong not only to singleton clusters of the solution set $\Omega$, but also to any subset of $\Omega$ (i.e. \textit{meta-clusters}) with different belief masses. This additional flexibility over FCM provides a deeper insight into the data and improves robustness against outliers. 
However, ECM was originally designed for tabular data  and it handles only quantitative attributes, i.e. data in a vector space. %
Indeed, ECM is based on Euclidean distance: the problem formulation assumes the capability to compute barycenters as a mean position of cluster elements. However, this is not always feasible. 
In order to adapt ECM to different types of data, variants have been proposed. Among these variants, ECMdd~\cite{softecm05} and RECM~\cite{softecm14} use relational similarity measures, while CECM~\cite{softecm15} introduces clustering constraints to better guide data partitioning and employs a specific Mahalanobis distance for each cluster. ECM+~\cite{albert2025ecm+} improves the center definition when Mahalanobis distances are used. Other variants, such as CatECM~\cite{softecm06}, are designed for categorical data. 
Credal C-Means (CCM)~\cite{softecm18} and Belief C-Means (BCM)~\cite{softecm17} modify how objects are assigned to meta-clusters. These approaches reconsider the computation of cluster centers by considering distances between objects and clusters, avoiding the arithmetic mean approximation, which improves accuracy for highly uncertain data.
Furthermore, some data types are not tabular, such as time series for which hard clustering methods have already been adapted~\cite{petitjean2014dynamic}. To the best of our knowledge, no evidential clustering methods have yet been proposed to handle such data. 

In this paper, we propose \SECM, a new algorithm that reformulates the ECM problem as a constrained multi-class clustering, where singleton clusters and meta-clusters (imprecise clusters) are individually optimized. In contrast to the classical ECM approach, meta-clusters are here treated like singleton clusters, allowing greater flexibility in dealing with uncertainties. 
A new constraint on meta-cluster barycenters ensures that they are positioned consistently with those of singleton clusters. In this way, \SECM is more flexible, and allows the use of non-Euclidean dissimilarity measures. This makes it possible to deal with complex data, such as mixed data (tabular data with both quantitative and qualitative features) or even non-tabular data such as time series.

\section{Background of Evidential C-Means (ECM)}\label{etat-art}
ECM is an algorithm based on the notions of uncertainty and imprecision defined in the theory of belief functions. This theory, also known as Dempster-Shafer theory~\cite{shafer1976mathematical}, provides powerful tools for modeling partial information. Considering $\omega$ as a variable defined in finite set $\Omega = \left \{ \omega_1, \dots, \omega_c \right \}$, partial knowledge of the true value of $\omega$ can be represented by a mass function $m$, defined from power set $2^{\Omega}$ to $\left[ 0, 1 \right]$ and satisfying: $\sum_{A \subseteq \Omega} m(A) = 1$. The quantity $m(A)$ is then interpreted as the amount of belief assigned to any subset $A \subseteq \Omega$, and $A$ is called a focal element if $m(A) > 0$. When $m(A) = 1$, we have a certain belief. When $m(A) > 0$ and $|A| > 1$, we have an imprecise belief.

For the ECM clustering problem \cite{softecm00}, Masson and Den{\oe}ux consider $\bm{X} = \left\{ \bm{x}_1, \dots, \bm{x}_n \right \}$ a set of $n$ individuals described in $\mathbb{R}^{p}$ to be grouped into a set of $c$ ``pure'' or ``singleton'' clusters, denoted $\Omega = \left\{ \omega_1, \dots, \omega_c \right \}$. 
The result of ECM is then the assignment of each example $\bm{x}_i\in\bm{X}$ to a subset $A \subseteq \Omega$ by a mass $m_i(A)\in\mathbb{R}_+$. When $A$ has a cardinality $|A|>1$, an example is classified imprecisely to all clusters of this set. In this case, $A$ is designated as a \textit{meta-cluster}.

ECM minimizes the following intra-cluster inertia:\footnote{Refer to the paper \cite{softecm00} for the minimization scheme.}
\begin{equation}\label{eq:ECMfunction}
    \begin{aligned}
    J_{ECM}(\bm{M},\bm{\mathcal{V}}) = & \sum_{i=1}^{n}\sum_{A\subseteq \Omega \setminus \emptyset} {\left | A \right |}^{\alpha} m_{i}(A)^{\beta} d^{2}(\bm{x}_i, \overline{\bm{v}}_A) + \sum_{i=1}^{n} \delta^{2} m_{i}(\emptyset)^{\beta}, \\
    \text{subject to} & \sum_{A \subseteq \Omega} m_{i}(A) = 1 \quad \forall i \in [n].
    \end{aligned}
\end{equation}
with $\bm{M}=\left\{m_i(A) \,\mid\, i\in [n],\,A\subseteq\Omega\right\}$ the credal partition, $\bm{\mathcal{V}}=\{\bm{v}_{\omega_j}\}_{j\in[c]}$ singleton clusters centroids, $\{\overline{\bm{v}}_{A}\}_{A\subseteq\Omega}$ the centroids of clusters and meta-clusters; and $d$ the Euclidean distance.

Fig.~\ref{fig:ECM} illustrates a soft clustering, ECM considers only the green dashed lines in its optimization problem and the position of the meta-cluster is computed as the barycenter of the centroids of singleton clusters. 

\begin{figure}[htbp]
\centering
\includegraphics[width=.8\columnwidth]{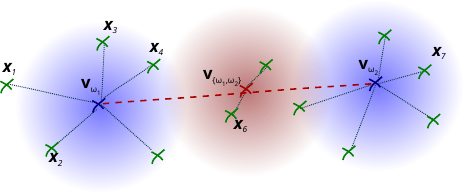}
\caption{Illustration of a 2-cluster evidential clustering. Examples are represented in green. Points in blue $\bm{v}_{\omega_1}$ and $\bm{v}_{\omega_2}$ indicate cluster centers, ${\omega_1}$ and ${\omega_2}$. The red point $\bm{v}_{\{\omega_1,\omega_2\}}$ is the center of meta-cluster $\{\omega_1,\omega_2\}$. This point is both the barycenter of examples belonging in the meta-cluster $\{\omega_1,\omega_2\}$; and the barycenter of $\bm{v}_{\omega_1}$ and $\bm{v}_{\omega_2}$. The fine dotted lines indicate the distances considered in computing the new clustering inertia.}
\label{fig:ECM}
\end{figure}

The hyperparameter $\beta>0$ controls the spread of the final partition: as the value of $\beta$ increases, the mass functions of the partition have degrees of belief distributed over all subsets $A \subseteq \Omega$. The coefficient $|A|^\alpha$ controls the imprecision of the credal partition. If $\alpha >0$ is large, then subsets with high cardinalities are penalized more than singletons. Thus, the degrees of belief are increased for singletons and decreased for other subsets. Finally, the hyperparameter $\delta$ corresponds to the minimal atypicity distance. When the distance between an example and the empty set is above this threshold, it is considered atypical.

It is worth noting that, in ECM, the centroids of the meta-clusters are not in $\bm{\mathcal{V}}$. The centroid of a meta-cluster $A \subseteq \Omega$, $|A|>1$, is computed as the isobarycenter of the centroids associated with the singleton clusters:
\begin{equation}\label{eq:barycentreECM}
\overline{\bm{v}}_A = \frac{1}{\left | A \right |}\sum_{\omega_k\in A} \mathbf{v}_{\omega_k}.
\end{equation}
This definition of a barycenter is correct when the data are in a Euclidean space. Nonetheless, it cannot be use for complex data such as categorical, non-numerical or non-tabular data, for which we might wish to use non-euclidean dissimilarity measures.\footnote{The properties of Euclidean distance in required by ECM to define both the function to be optimized~\eqref{eq:ECMfunction} and the meta-cluster centroids~\eqref{eq:barycentreECM}.}

\section{Soft Evidential C-Means (\SECM) Approach}\label{soft-ecm}
We propose to reformulate the ECM problem in a more generic way, so that the evidential clustering problem can be solved with dissimilarity measures other than Euclidean distance. Intuitively, our approach relaxes the definition of meta-cluster centroid as the explicit isobarycenter of singletons' centroids. In Fig.~\ref{fig:ECM}, the meta-cluster centroid $\bm{v}_{\{A,B\}}$ is constrained by both being as close as the barycenter of its examples (green dashed lines) and being as close as barycenter of the singletons clusters centroids (red dashed lines).

In this section, $\bm{X}=\{\bm{x}_1, \dots , \bm{x}_n\}$ denotes a set of objects in $\mathbb{X}$ (e.g., time series) and $d$ denotes a semi-metric between these objects.\footnote{A semi-metric is a symmetric, positive and defined dissimilarity measure, but it does not necessarily hold the triangular inequality.} In a classical way~\cite{petitjean2014dynamic}, we can then define the barycenter of a set of objects as an element of $\mathbb{X}$ that minimizes inertia.
More formally, the barycenter, $\mathbf{v}_E$, of the subset $E\subseteq\bm{X}$ is defined as follows:
\begin{equation}\label{eq:barycentresoftECM}
\mathbf{v}_E = \argmin_{\bm{x}\in\mathbb{X}}\sum_{\bm{e}\in E}d(\bm{e}, \bm{x}).
\end{equation}

Then, the \SECM problem is formalized as the following constrained minimization problem:
\begin{equation}\label{eq:SoftECMfunction}
\begin{aligned}
 \min_{\bm{M},\bm{\mathcal{V}}} J^{\alpha,\beta,\delta}(\bm{M},\bm{\mathcal{V}})  = & \sum_{i=1}^n \sum_{A \subseteq \Omega\setminus\emptyset} |A|^\alpha m_{i}(A)^\beta d(\mathbf{x}_i,\mathbf{v}_A) + \sum_{i=1}^n \delta^2 m_{i}(\emptyset)^\beta\\ 
 \text{subject to} & \sum_{A \subseteq \Omega} m_{i}(A) = 1,\; \forall i \in [n],\\
 & \mathbf{v}_A = \argmin_{\mathbf{x}\in\mathbb{X}}\sum_{\omega_k\in A}d(\mathbf{v}_{\{\omega_k\}}, \mathbf{x}),\, \forall A\subseteq\Omega.
\end{aligned}
\end{equation}

In this formulation, $\bm{M} = \left(m_i\left(A\right)\right)_{i\in[n]}$ represents the mass functions on a focal element~$A \subseteq \Omega$, while $\bm{\mathcal{V}} = \{\mathbf{v}_A\}_{A\subset\Omega\setminus\emptyset}$ corresponds to the barycenters of all singleton clusters and meta-clusters. In contrast to ECM formulation \eqref{eq:ECMfunction}, which focused on the optimization of cluster positions only and defined the positions of meta-clusters by~\eqref{eq:barycentreECM}, here we optimize both. 
The second constraint enforces the centroid of each meta-cluster $A$ to be the barycenter of the cluster centroids ($\bm{v}_{\omega_k}$ where $\omega_k\in A$). We find again the expression of the equation that defines the barycenter (see~\eqref{eq:barycentresoftECM}).

\begin{prop}\label{res:equivalence}
Provision $\mathbb{X}=\mathbb{R}^p$ and $d$ is the squared Euclidean distance, the problems of \eqref{eq:ECMfunction} and \eqref{eq:SoftECMfunction} are equivalent.
\end{prop}

\begin{proof}
From the definition of the notion of barycenter in Equation~\ref{eq:barycentresoftECM}, and in the case where $d$ is the Euclidean distance, we have the equivalence between:
\begin{enumerate}
\item $\mathbf{v}_A = \argmin_{x\in\mathbb{X}}\sum_{\omega_k\in A}d(\mathbf{v}_{\{\omega_k\}}, x)$ (inertia minimization),
\item $\mathbf{v}_A$ is the centroid of clusters or meta-clusters (by the properties of barycenters).
\end{enumerate}

\vspace{5pt}

We now define the mapping $\widebar{\cdot}: \mathbb{R}^{p\times c} \mapsto \mathbb{R}^{p\times 2^c}$ which constructs a set of centroids for all clusters and meta-clusters from the centroids of pure clusters.
Its inverse, denoted $\underline{\cdot}: \mathbb{R}^{p\times 2^c} \mapsto \mathbb{R}^{p\times c}$, selects only the centroid corresponding to pure a cluster from the set of all centroids. 
Then, for any $\bm{M}$ and $\bm{\mathcal{V}}$, we have:
\begin{equation}\label{preuve:eq:transfo}
\begin{array}{l}
J_{ECM}(\bm{M},\bm{\mathcal{V}}) =J^{\alpha,\beta,\delta}\left(\bm{M},\widebar{\bm{\mathcal{V}}}\right)\\
J^{\alpha,\beta,\delta}\left(\bm{M},\bm{\mathcal{V}}\right)=J_{ECM}\left(\bm{M},\underline{\bm{\mathcal{V}}}\right) \\
\end{array}
\end{equation}

This simply expresses the fact that, although the parameters of $J_{ECM}$ only take into account the centroids of pure clusters, the objective function actually involves the centroids of both pure clusters and meta-clusters (when computed using the Euclidean distance).

It is then easy to verify that $\widebar{\underline{\mathcal{V}}}=\mathcal{V}$.

\vspace{5pt}

Having introduced these notions, we can now prove the proposition above by contradiction.

\textbf{Let us start by proving that \eqref{eq:ECMfunction} $\Leftarrow$ \eqref{eq:SoftECMfunction}.} Suppose that $\bm{M}^*$ and $\bm{\mathcal{V}}^*$ are the optimum of $J^{\alpha,\beta,\delta}(\bm{M},\bm{\mathcal{V}})$ as defined in Equation~\ref{eq:SoftECMfunction}, and assume that there exists a better solution than $\left(\bm{M}^*,\underline{\bm{\mathcal{V}}^*}\right)$ for minimizing $J_{ECM}$. Let $\left(\bm{M}',\bm{\mathcal{V}}'\right)$ denote such a solution, i.e.,
$J_{ECM}\left(\bm{M}',\bm{\mathcal{V}}'\right)<J_{ECM}\left(\bm{M}^*,\underline{\bm{\mathcal{V}}^*}\right).$
Then, using the equalities from (\ref{preuve:eq:transfo}) on both sides of the inequality, we obtain:
$J^{\alpha,\beta,\delta}\left(\bm{M}',\widebar{\bm{\mathcal{V}}}'\right)<J^{\alpha,\beta,\delta}\left(\bm{M}^*,\widebar{\underline{\bm{\mathcal{V}}}}^*\right),$
and therefore:
$J^{\alpha,\beta,\delta}\left(\bm{M}',\widebar{\bm{\mathcal{V}}}'\right)<J^{\alpha,\beta,\delta}\left(\bm{M}^*,\bm{\mathcal{V}}^*\right).$

Moreover, we can verify that $\left(\bm{M}',\widebar{\bm{\mathcal{V}}}'\right)$ is a valid solution of Equation~\ref{eq:SoftECMfunction}. Indeed:
The constraints on the sum of masses are trivially satisfied since $\left(\bm{M}',\bm{\mathcal{V}}'\right)$ was a solution of Equation~\ref{eq:ECMfunction}, which includes these constraints on $\mathcal{M}$.
The second set of constraints is also satisfied, because by construction of $\widebar{\bm{\mathcal{V}}}'$, the additional elements are barycenters of the pure clusters, and in the Euclidean case, such barycenters are the points that minimize inertia (see above).

Thus, $\left(\bm{M}',\widebar{\bm{\mathcal{V}}}'\right)$ would be a better solution than the supposed optimum of Soft-ECM, which is impossible. We conclude that there exists no better solution than $\left(\bm{M}^*,\underline{\bm{\mathcal{V}}^*}\right)$ for $J_{ECM}$.

\vspace{10pt}

\textbf{Let us new prove that \eqref{eq:ECMfunction} $\Rightarrow$ \eqref{eq:SoftECMfunction}.} The reverse direction can be shown symmetrically.

\vspace{10pt}

We thus conclude that the two minimization problems are equivalent in the Euclidean setting. \\\null\hspace{14.8cm}$\square$
\end{proof}

\section{Optimization Scheme for \SECM}
The \SECM problem is a bi-level optimization problem~\cite{dempe2020bilevel}, i.e. an optimization problem in which parameters are themselves defined by another optimization problem.

We propose a heuristic that relaxes the hard constraint of identifying the meta-clusters barycenters into an additional numerical constraint in the global minimization function. The relaxed minimization problem of \SECM then becomes:
\begin{equation}\label{eq:softECMrelaxed}
    \begin{split}
        \min_{\bm{M},\bm{\mathcal{V}}} &\; \sum_{i=1}^n \sum_{A \subseteq \Omega \setminus \emptyset} |A|^\alpha m_{i}(A)^\beta d(\mathbf{x}_i,\mathbf{v}_A) 
        + \sum_{i=1}^n \delta^2 m_{i}(\emptyset)^\beta  \\
        &+ \lambda \sum_{A \subseteq \Omega \setminus \emptyset} \sum_{\omega_k \in A} d(\mathbf{v}_{\{\omega_k\}}, \mathbf{v}_A), \\
        \text{subject to} &\; \sum_{A \subseteq \Omega} m_{i}(A) = 1, \quad \forall i \in [n].
    \end{split}
\end{equation}
While the first and second terms are similar to those of ECM, the third term imposes the additional constraint of consistency between the barycenters of the $(\{\omega_{k}\})_{k\in[c]}$ clusters and those of meta-clusters. The new hyperparameter~$\lambda\geq 0$ adjusts the importance given to this constraint. The higher $\lambda$, the more the solution respects the inter-clusters positioning constraint.

We can use an alternating optimization scheme to solve the relaxed \SECM optimization problem (see Algorithm~\ref{algorithm:algo1}), with the following two steps alternating (until convergence or a fixed number of iterations):

\textbf{1) Optimization of $\bm{M}$, with $\bm{\mathcal{V}}$ fixed.} 
The credal partition $\bm{M}$, which minimizes the objective function, is computed using the expression below:
\begin{equation}\label{eq:partitionM}
    \begin{aligned}
    & m_i(A) = \frac{|A|^{-\frac{\alpha}{\beta-1}} d\left(\mathbf{x}_i,\mathbf{v}_A\right)^{-\frac{1}{\beta-1}}}{\sum_{B\neq \emptyset} |B|^{-\frac{\alpha}{\beta-1}} d\left(\mathbf{x}_i,\mathbf{v}_B\right)^{-\frac{1}{\beta-1}} + \delta^{-\frac{2}{\beta-1}}}, \\
    & m_i(\emptyset) = 1-\sum_{A\neq \emptyset}m_i(A).
    \end{aligned}
\end{equation}
This solution was proposed for ECM~\cite{softecm00} and is still valid for \SECM. Indeed, the term added to the objective function does not depend on $\bm{M}$, so the optimization with respect to $\bm{M}$ of the problem in~\eqref{eq:softECMrelaxed} is the same as the problem in~\eqref{eq:ECMfunction}.

\textbf{2) Optimization of $\bm{\mathcal{V}}$, with $\mathbf{M}$ fixed.} 
The optimization of~\eqref{eq:softECMrelaxed} when $\bm{M}$ is fixed is equivalent to the following unconstrained optimization problem:
\begin{equation}\label{eq:SECMwrtV}
 \argmin_{\mathcal{V}} \sum_{A \subseteq \Omega\setminus\emptyset}\left(\sum_{i=1}^n |A|^\alpha m_{i}(A)^\beta d(\mathbf{x}_i,\mathbf{v}_A) + \lambda \sum_{\omega_k\in A}d(\mathbf{v}_{\{\omega_k\}}, \mathbf{v}_A)\right).
\end{equation}

We adopt a different approach from ECM and use a numerical method (gradient descent) to find an approximate solution. The use of optimization tools based on \textit{automatic differentiation} allows to solve this problem efficiently. We note that the optimization of this function requires the differentiability of~$d$.

\SECM converges after a finite number of iterations as $J^{\alpha,\beta,\delta,\lambda}$ decreases at each step and is a positive function. 
Indeed, updating the credal partition $\mathbf{M}$ with Lagrange multipliers does not increase the objective function $J^{\alpha,\beta,\delta,\lambda}$~\cite{softecm00}. 
Regarding the optimization with respect to $\bm{\mathcal{V}}$, the gradient descent ensure to not increase the objective function.

\begin{algorithm}[htbp]
\SetAlgoLined
\KwIn{$\left \{ \mathbf{x}_1, \dots , \mathbf{x}_n \right \}$: $n$ objects in $\mathbb{X}^{n}$.}
\Parameters{ \\       
    \hspace*{0.8em}
    \begin{minipage}[t]{\linewidth}
        $c$: number of clusters \\
        $\alpha>0, \beta > 1$: weightings for belief \\
        $\delta > 0$: threshold for outliers \\
        $\lambda>0$: consistency importance between barycenters\\
        $\epsilon, \xi > 0$: convergence thresholds \\
        $\rho$: learning rate
    \end{minipage}}
 \vspace{.8em}
 \Initialize{ \\
 \hspace*{0.8em}
    \begin{minipage}[t]{\linewidth}
         $\mathbf{v}^{0}$: Initialize $c$ prototypes randomly \\
         $\mathbf{v}^{0}_{A}$: Compute the mean of $\mathbf{v}^{0}$ for $\left | A \right | > 1$  
    \end{minipage}}
 \vspace*{0.8em}
 $t \gets 0$ \\
 \Do{$\|\mathbf{M}^{t} - \mathbf{M}^{t-1}\| < \epsilon$}{
    (1) Compute the credal partition $\mathbf{M}$ using \eqref{eq:partitionM}\;
    (2) Update the prototype vector: \\
    $\theta \gets 0$ \\
    \Do{$\|\mathcal{V}^{\theta+1} - \mathcal{V}^{\theta}\| < \xi$}{
         $\mathcal{V}^{\theta+1} = \mathcal{V}^{\theta} - \rho \nabla_{\mathcal{V}} J^{\alpha,\beta,\delta,\lambda}(\mathbf{M},\mathcal{V}^{\theta})$\;
         $\theta \gets  \theta + 1$\;
     }
 $t \gets t + 1$\;
 }
 \vspace*{0.8em}
 \KwOut{Optimal solution $\mathbf{M}$ and $\mathcal{V}$.}
 \caption{\SECM Algorithm}
 \label{algorithm:algo1}
\end{algorithm}

\section{Experiments}\label{empirical-results}
In this section, we present the results of several experiments. First, we demonstrate the ability of \SECM to perform clustering comparable to ECM on a classic toy problem. Next, we compare various fuzzy clustering methods with our approach using real-world datasets of different types, including tabular data (numerical and categorical) and time series. Finally, we showcase the clustering of synthetic time series to illustrate how our approach addresses the challenges of fuzzy clustering on complex data with a semi-metric.

\subsection{Experiments with Synthetic Data}\label{sec:expe_donnees-synthetic}
We compare the behavior of \SECM and ECM on a classic dataset, \textit{Diamond}~\cite{softecm00}. It consists of twelve objects (see Fig.~\ref{fig:exp1_1}), where objects 1~to 11~being normal, and object 12 is an outlier. \SECM and ECM were run with the same parameters: $\alpha = \frac{1}{6}$, $\beta = 2$, $\delta = 11$ and, for \SECM, $\lambda = 1.5$. The data were partitioned into $c = 2$ clusters. Fig.~\ref{fig:exp1_2} and Fig.~\ref{fig:exp1_3} depict the $2^c$ cluster masses (focal elements) ---$m(\omega_1)$, $m(\omega_2)$, $m(\Omega)$, and $m(\emptyset)$--- for ECM and \SECM, respectively.

\begin{figure*}[tbp]
    \centering 
    \begin{subfigure}[b]{0.48\textwidth}
        \includegraphics[width=\textwidth]{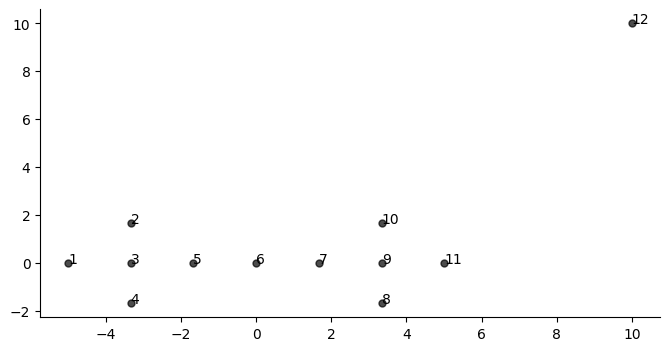}
        \caption{Données initiales}
        \label{fig:exp1_1}
    \end{subfigure}
    \hfill
    \begin{subfigure}[b]{0.48\textwidth}
        \includegraphics[width=\textwidth]{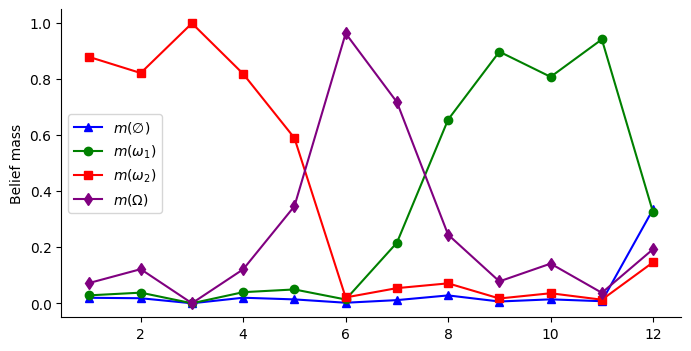}
        \caption{ECM}
        \label{fig:exp1_2}
    \end{subfigure}

    \begin{subfigure}[b]{0.48\textwidth}
        \includegraphics[width=\textwidth]{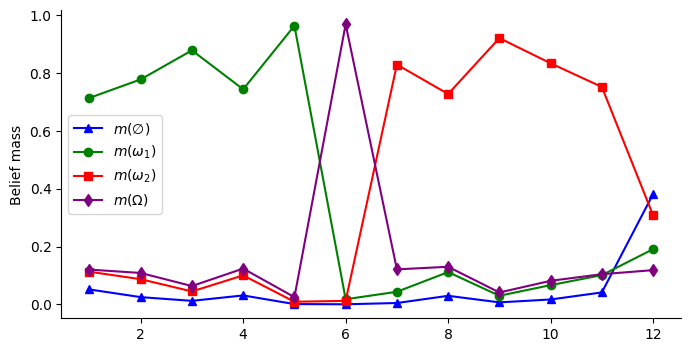}
        \caption{\SECM ($\lambda=1.5$)}
        \label{fig:exp1_3}
    \end{subfigure}
    \hfill
    \begin{subfigure}[b]{0.48\textwidth}
        \includegraphics[width=\textwidth]{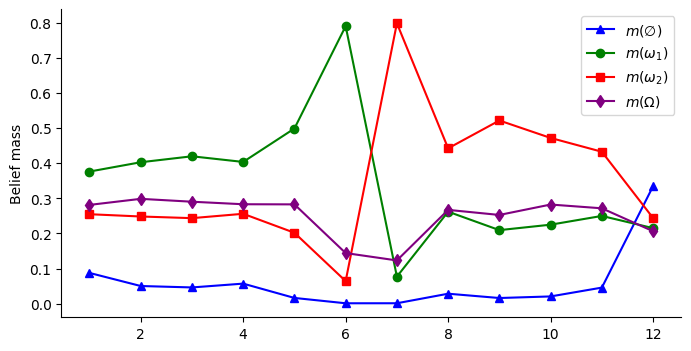}
        \caption{\SECM ($\lambda=3.5$)}
        \label{fig:exp1_4}
    \end{subfigure}
    \caption{Illustration of the Diamond dataset clustering \cite{softecm00}. (a) illustration of raw data. (b), (c) and (d) mass distributions for each example for ECM and \SECM. A high mass indicates cluster or meta-cluster membership.}
    \label{fig:exp1}
\end{figure*}

Objects 1 to 5 and 8 to 11 are assigned to two distinct clusters by ECM and \SECM, while object~12 is identified as an outlier. For objects~6 and~7, ECM assigns most of the mass to the ignorance (meta-cluster) $\Omega$, as they are located near the centroid of $\Omega$, which is defined as the average of the centroids of $\omega_1$ and $\omega_2$ (see Fig.~\ref{fig:exp1_2}). In contrast, \SECM optimizes this centroid more effectively, with most belief mass assigned to $\omega_2$ for object 7 (cf. Fig.~\ref{fig:exp1_3}). When we increase $\lambda$, until $3.5$ with this dataset, it further constrains inter-cluster positioning. As shown in Fig.~\ref{fig:exp1_4}, object 6 was then classified into the $\omega_1$, consistent with the results obtained by the authors in (BCM)~\cite{softecm17} and aligned the raw data.

\subsection{Experiments and Comparisons with Real-World Data}\label{sec:expe_donneesreels}
\subsubsection{Experimental Setting}\label{sec:expe_setting}
We evaluate here the performance of \SECM against ECM and its state-of-the-art variants, using nine publicly available real-world datasets.\footnote{\url{https://archive.ics.uci.edu}, \url{https://www.timeseriesclassification.com/}} The datasets are of different types (numerical, categorical or time series). Tab.~\ref{tab:data} summarizes the characteristics of the datasets.

\begin{table}[tb]
\renewcommand{\arraystretch}{1.05}
\begin{center}
\begin{tabular}{l@{\hskip 0.08in}c@{\hskip 0.05in}c@{\hskip 0.05in}c@{\hskip 0.1in}c@{\hskip 0.05in}c@{\hskip 0.1in}c}
\toprule
\textbf{Data} & $n$ & $p$ & $c$ & \textbf{Length} & \textbf{Metric(s)} & \textbf{Inputs}\\
\midrule
Abalone & $4\,177$ & $8$ & $3$ & - & Euclidean & Numerical\\
Ecoli & $336$ & $7$ & $8$ & - & Euclidean & Numerical\\
Glass & $214$ & $9$ & $6$ & - & Euclidean & Numerical\\
Breast Cancer (BC) & $286$ & $9$ & $2$ & - & Hamming & Categorical\\
Lung & $32$ & $56$ & $3$ & - & Hamming & Categorical\\
Soybean & $47$ & $35$ & $4$ & - & Hamming & Categorical\\
ERing & $30$ & $4$ & $6$ & $65$ & Euclidean, Soft-DTW & Time series \\
BasicMotions (BM) & $40$ & $6$ & $4$ & $100$ & Euclidean, Soft-DTW & Time series\\
AtrialFibrillation (AF) & $15$ & $2$ & $3$ & $640$ & Euclidean, Soft-DTW & Time series \\
\bottomrule
\end{tabular}
\end{center}
\caption{Experimental dataset characteristics, $n$: number of objects, $p$: number of dimensions and $c$: number of clusters.}
\label{tab:data}
\end{table}

\SECM is implemented in Python and based on PyTorch~\cite{paszke2017automatic} for solving~\eqref{eq:SECMwrtV}.\footnote{Code available: \url{https://gitlab.inria.fr/tguyet/ecm-ts}.} Benchmark algorithms used to evaluate the performance of our algorithm include ECM for numerical data, Evidential C-Means for categorical data (CatECM~\cite{softecm06}), and Evidential C-Medoids (ECMdd~\cite{softecm05}) for the relational format of data. Additionally, we compare our results with a hard clustering approach, specifically K-means and its variants K-modes and TimeSeriesKMeans (TSKmeans~\cite{softecm12}).

Each algorithm was executed 10 times for each of the following parameter settings: $\alpha=2.0$, $\delta=10$, $\beta \in \{1.1, 1. 2, \dots,$ $2.0 \}$, $\lambda \in \{1.0, 2.0, \dots, 10.0 \}$. 
Each run was was stopped upon convergence to a stationary value with $\epsilon=10^{-3}$.
The number of subsets for meta-clusters was limited to $2$ for the sake of interpretability and computation time. For each dataset, the number of clusters corresponds to the prior number of clusters. For credal partitions, the pignistic~\cite{softecm00} transformation was applied to convert them into hard partitions. Euclidean, Soft-DTW and Hamming metrics were used for clustering numerical, temporal and categorical data respectively (see Tab.~\ref{tab:data}). They are all differentiable.

\subsubsection{Analysis of $\lambda$ Parameter}\label{sec:parameters}
In this section, we examine the sensitivity of \SECM results to the new parameter $\lambda$ on each dataset. We use the average normalized specificity as the internal validity index of a credal partition, as defined by~\eqref{eq:normalizedSpecificity}, following the approach proposed in~\cite{softecm00}, where $0 \le N^\star(c) \le 1$. 

\begin{equation}\label{eq:normalizedSpecificity}
N^\star(c) 
= \frac{1}{n \log_2(c)} 
\sum_{i=1}^{n}
  \sum_{A \in 2^{\Omega} \setminus \emptyset} 
    m_i(A) \,\log_2 \lvert A \rvert
\end{equation}

Fig.~\ref{fig:pram_beta_lambda_ri} illustrates the normalized specificity ($N^\star$) computed for \SECM with several combinations of $\lambda$ and $\beta$ values on real-world datasets.  
We can observe that the values of $\lambda$ and $\beta$ have a significant impact. Selecting appropriate values for $\lambda$ on each dataset improves the performance of \SECM. 
These first results show no evident trend in the evolution of $N^\star$ and, at this stage, we recommend a tuning of the $\lambda$ parameter for each dataset, for instance via a grid search. 

\begin{figure}[tb]
\centerline{\includegraphics[width=\columnwidth]{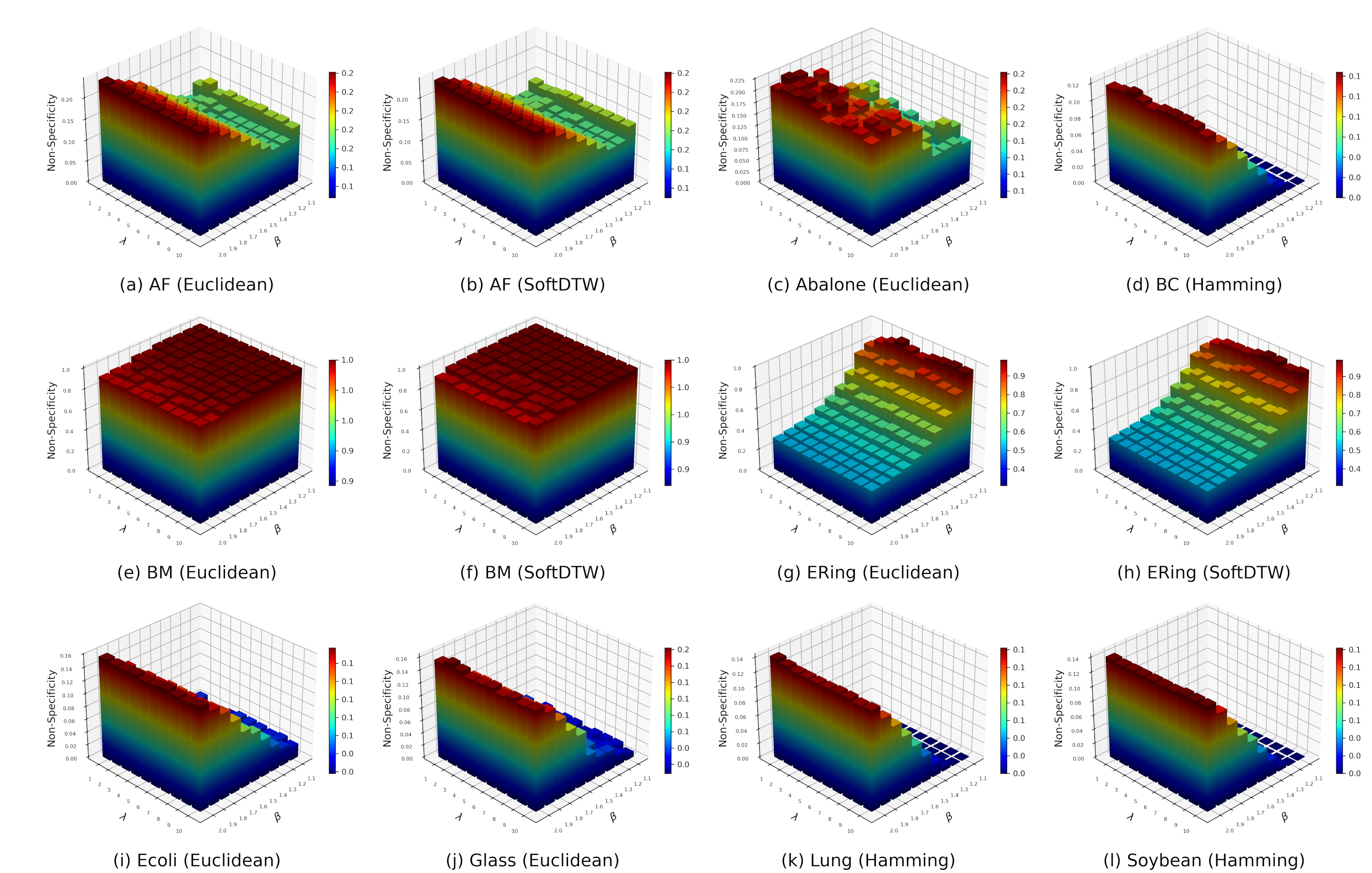}}
\caption{Normalized specificity ($N^*$) of \SECM wrt parameters $\beta$ and $\lambda$ on 9 real-world datasets.}
\label{fig:pram_beta_lambda_ri}
\end{figure}

\subsubsection{Comparison with other Algorithms on Real-World Datasets}\label{sec:comp_donneesreels}
For each dataset, the values of parameters $\beta$ and $\lambda$ that minimize the non-specificity ($N^\star$) are selected for the credal partitions. Next, we compute the mean and standard deviation ($\sigma$) of two standard performance measures ---accuracy and Rand Index (RI)~\cite{softecm19}--- across 10~executions. 

Tab.~\ref{tab:resultats_ri} shows the performances obtained for each algorithm on each dataset. The accuracy is calculated using a prior known class of data. 
Indeed, for each algorithm and each dataset, we calculated the mean and standard deviation statistics in mean($\sigma$) format in Tab.~\ref{tab:resultats_ri} (according to parameters defined in Sec.~\ref{sec:expe_setting} and those selected in Sec.~\ref{sec:parameters}).

\begin{table*}[tb]
\centering
\resizebox{\linewidth}{!}{%
\begin{tabular}{lccccccc|}
    \toprule
     & \rotatebox[origin=c]{30}{\textbf{ECM}} 
     & \rotatebox[origin=c]{30}{\textbf{ECMdd}} 
     & \rotatebox[origin=c]{30}{\textbf{CatECM}} 
     & \rotatebox[origin=c]{30}{\textbf{\SECM}} 
     & \rotatebox[origin=c]{30}{\textbf{K-Means}} 
     & \rotatebox[origin=c]{30}{\textbf{K-Modes}} 
     & \rotatebox[origin=c]{30}{\textbf{TSKmeans}} \\
    \midrule
    Abalone & \textbf{0.58 (0.00)} & 0.58 (0.02) & - & 0.57 (0.03) & 0.55 (0.02) & - & - \\
    Ecoli & 0.89 (0.004) & 0.83 (0.05) & - & 0.81 (0.06) & \textbf{0.89 (0.02)} & - & - \\
    Glass & 0.64 (0.05) & 0.63 (0.06) & - & \textbf{0.66 (0.005)} & 0.64 (0.03) & - & - \\
    BC & - & 0.58 (0.00) & 0.58 (0.00) & \textbf{0.58 (0.01)} & - & 0.58 (0.00) & - \\
    Lung & - & 0.52 (0.06) & 0.55 (0.06) & 0.48 (0.12) & - & \textbf{0.60 (0.00)} & - \\
    Soybean & - & 0.79 (0.11) & 0.91 (0.06) & \textbf{0.91 (0.02)} & - & 0.91 (0.12) & - \\
    ERing (Eucl.) & - & 0.85 (0.04) & - & 0.78 (0.04) & - & - & \textbf{0.91 (0.02)} \\
    BM (Eucl.) & - & 0.48 (0.10) & - & \textbf{0.57 (0.09)} & - & - & 0.43 (0.05) \\
    AF (Eucl.) & - & 0.42 (0.04) & - & 0.50 (0.07) & - & - & \textbf{0.40 (0.05)} \\
    ERing (SoftDTW) & - & \textbf{0.90 (0.03)} & - & 0.76 (0.02) & - & - & 0.86 (0.03)  \\
    BM (SoftDTW) & - & 0.56 (0.05) & - & \textbf{0.59 (0.09)} & - & - & 0.41 (0.04) \\
    AF (SoftDTW) & - & 0.47 (0.08) & - & \textbf{0.53 (0.09)} & - & - & 0.40 (0.01) \\
    \bottomrule
\end{tabular}
\begin{tabular}{ccccccc}
    \toprule \rotatebox[origin=c]{30}{\textbf{ECM}} 
     & \rotatebox[origin=c]{30}{\textbf{ECMdd}} 
     & \rotatebox[origin=c]{30}{\textbf{CatECM}} 
     & \rotatebox[origin=c]{30}{\textbf{\SECM}} 
     & \rotatebox[origin=c]{30}{\textbf{K-Means}} 
     & \rotatebox[origin=c]{30}{\textbf{K-Modes}} 
     & \rotatebox[origin=c]{30}{\textbf{TSKmeans}} \\
    \midrule 
    0.51 (0.00) & 0.50 (0.01) & - & \textbf{0.52 (0.02)} & 0.49 (0.03) & - & - \\
    0.82 (0.002) & 0.74 (0.05) & - & 0.71 (0.07) & \textbf{0.82 (0.03)} & - & - \\
    0.55 (0.04) &  0.54 (0.05) & - & \textbf{0.59 (0.002)} & 0.57 (0.02) & - & - \\
    - & 0.70 (0.00) & 0.70 (0.00) & \textbf{0.71 (0.01)} & - & 0.70 (0.00) & - \\
     - & 0.50 (0.03) & 0.49 (0.05) & 0.47 (0.07) & - & \textbf{0.53 (0.00)} & - \\
     - & 0.70 (0.10) & 0.87 (0.10) & 0.70 (0.12) & - & \textbf{1.00 (0.00)} & - \\
     - & 0.67 (0.08) & - & 0.56 (0.08) & - & - & \textbf{0.81 (0.05)} \\
     - & \textbf{0.47 (0.02)} & - & 0.41 (0.05) & - & - & 0.40 (0.02) \\
     - & 0.47 (0.02) & - & \textbf{0.51 (0.08)} & - & - & 0.46 (0.04) \\
     - & 0.70 (0.07) & - & 0.55 (0.03) & - & - & \textbf{0.80 (0.06)} \\
     - & 0.47 (0.06) & - & \textbf{0.47 (0.07)} & - & - & 0.40 (0.03) \\
     - & 0.49 (0.06) & - & \textbf{0.50 (0.08)} & - & - & 0.47 (0.00) \\
    \bottomrule
\end{tabular}
}
\caption{Clustering average performance (\textit{Rand Index} on the left and \textit{Accuracy} on the right) of algorithms on real datasets. Figures in parentheses are standard deviations. Bold numbers indicate best results by dataset. ``\textit{-}'' refers to inappropriate algorithms for a given dataset.}
\label{tab:resultats_ri}
\end{table*}

The results show that \SECM performs well across different data types, both in terms of RI and accuracy. For numerical data, \SECM achieves competitive results, often close to or better than \textit{ECM} and \textit{K-Means}, with an average RI of $56.6\%$ and accuracy of $51.6\%$. However, it shows slightly higher variability in performance, according to $\sigma$. On categorical data, \SECM outperforms \textit{K-Modes} and \textit{CatECM} in many cases, though its performance varies across iterations, with higher $\sigma$ values. For time series datasets, where both Euclidean and Soft-DTW metrics are used, \SECM demonstrates competitive performance, often surpassing \textit{TSKmeans}, especially with Soft-DTW metrics on Accuracy. Overall, \SECM performs well in about $83\%$ of cases across different data types, and proves its versatility across multiple data types and metrics.

In conclusion, this experimentation confirms that \SECM achieves comparable or superior performance to the benchmark algorithms used in this comparative experimentation. 
This is beyond the result of the Proposition~\ref{res:equivalence}, demonstrating the usefulness of a flexible problem formulation. 
In addition, it offers the advantage of a better fuzzy representation in uncertain data contexts.

\subsection{Experiments with Time Series Data}\label{sec:expe_timeseries}
Regarding time series datasets, a Euclidean metric and a semi-metric, the Dynamic Time Warping (DTW), are experimented. DTW is particularly useful for datasets with time series of different lengths or where there is temporal flexibility between instances of the same class. 
Since DTW is not a metric, new approaches have adapted clustering algorithms to DTW~\cite{petitjean2014dynamic}. However, to our knowledge, no method performs credal partitioning of time series using DTW.

In this section, we present the results of two experiments obtained on artificial time series analysis. They illustrate the relevance and ability of our algorithm to perform fuzzy partitioning of time series using DTW, or more specifically Soft-DTW~\cite{softecm04}, which is its differentiable version.

In a first experiment, we use classic synthetic time series, named \textit{Cylinder-Bell-Funnel}, illustrated in Fig.~\ref{fig:ts_expe2}. The dataset includes $50$ items of each class. These items are grouped according to index: from $1$ to $50$, from $51$ to $100$ and from $101$ to $150$, to ease interpretation of the mass plots. 
To illustrate the flexibility of \SECM, we applied it with two dissimilarity measures: the $L_2$ metric and the Soft-DTW semi-metric. 

The results are shown in Fig.~\ref{fig:ts_expe1}. The mass plots illustrate the expected partitioning for this dataset with $c=3$. When we use the $L_2$ metric, the algorithm fails to identify the correct clusters. Indeed, the examples are assigned equal masses for the three clusters, whatever the true behavior of the time series. In contrast, with Soft-DTW, each expected cluster is mainly assigned to a singleton cluster, representing a better clustering.
This experiment demonstrates that our algorithm effectively addresses the problem of fuzzy partitioning of time series, exploiting the advantages of Soft-DTW. 

\begin{figure}[tb]
\centering 
\includegraphics[width=0.42\textwidth, trim= 30 3 55 20, clip]{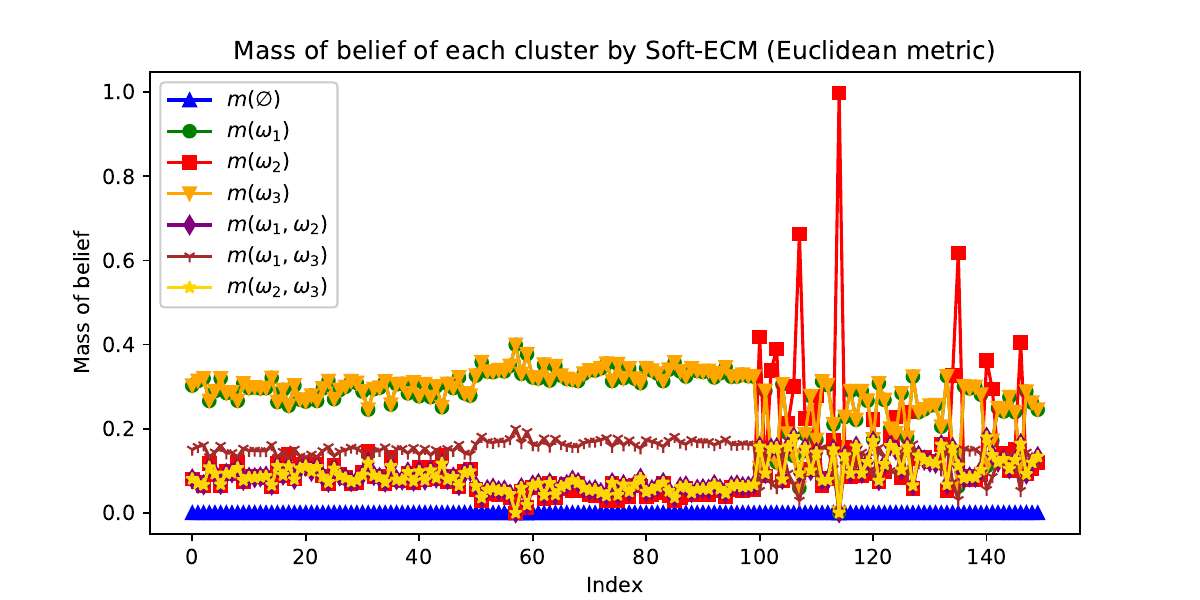}
\includegraphics[width=0.42\textwidth, trim= 30 3 55 20, clip]{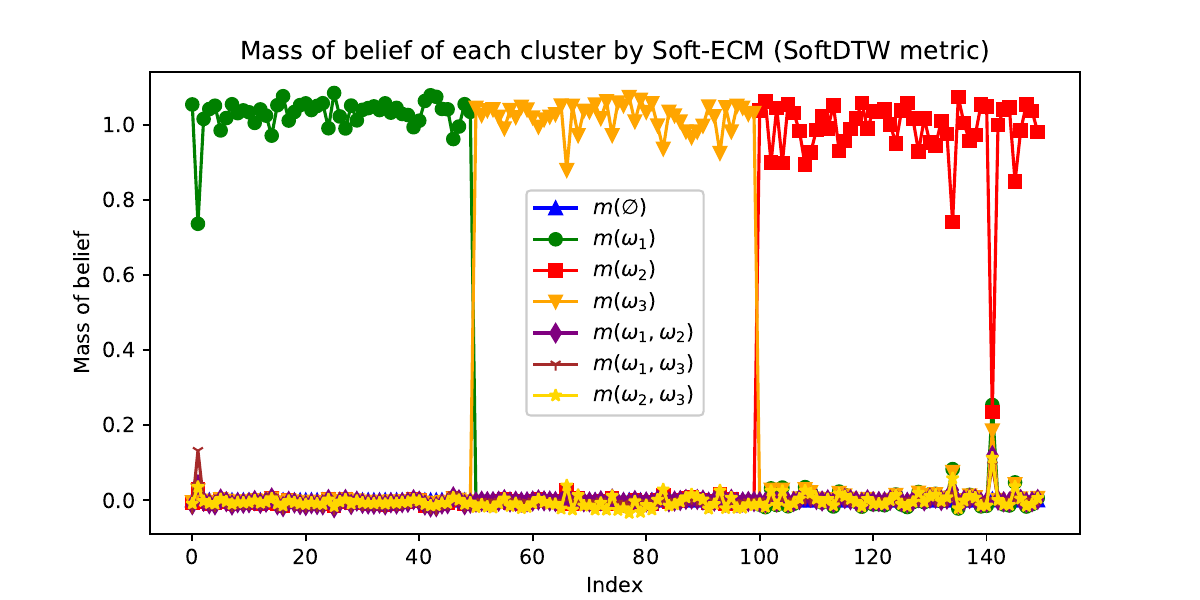}
\caption{On top, plot of mass obtained using $L_2$ distance. On bottom, mass graph obtained by Soft-DTW.}
\label{fig:ts_expe1}
\end{figure}

In a second experiment, we aim to demonstrate the advantages of fuzzy clustering for these time series. The previous task can be easily addressed using the DBA (DTW Barycenter Averaging) algorithm~\cite{petitjean2014dynamic}, but would not consider situations where time series could be fuzzily assigned to several classes because they mix two behaviors. 
In this experiment, we therefore used class \textit{Bell}, class \textit{Funnel} and a new class constructed by combining series \textit{Bell}+\textit{Funnel} ($M$-shaped time series).

\begin{figure}[tb]
\centering
\includegraphics[width=0.4\textwidth, trim= 30 3 55 20, clip]{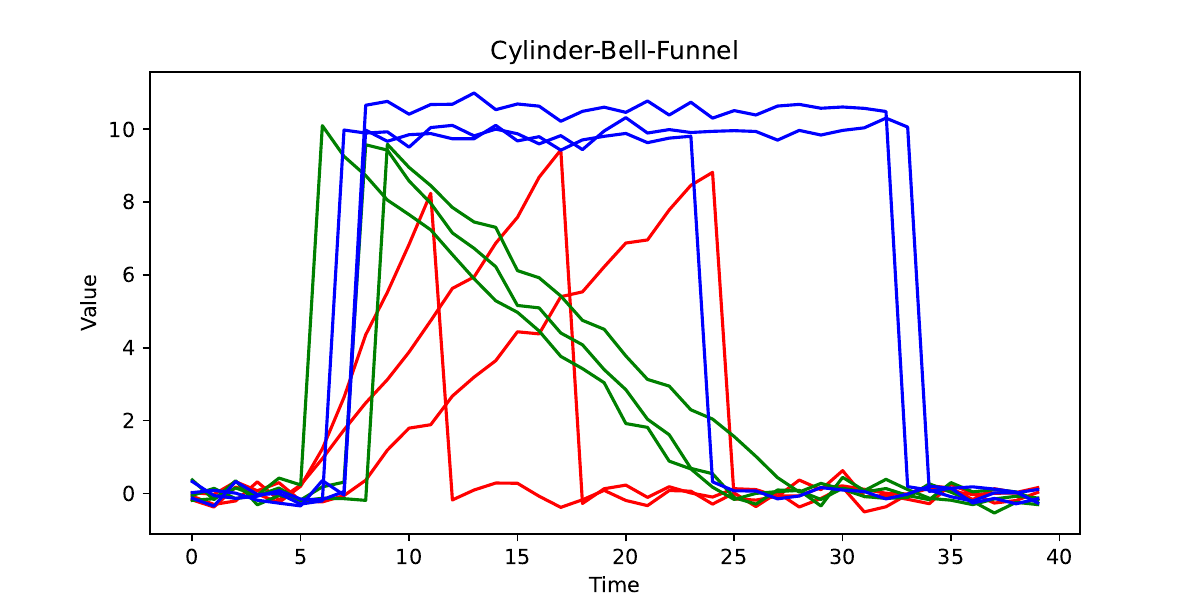}
\includegraphics[width=0.4\textwidth, trim= 30 3 55 20, clip]{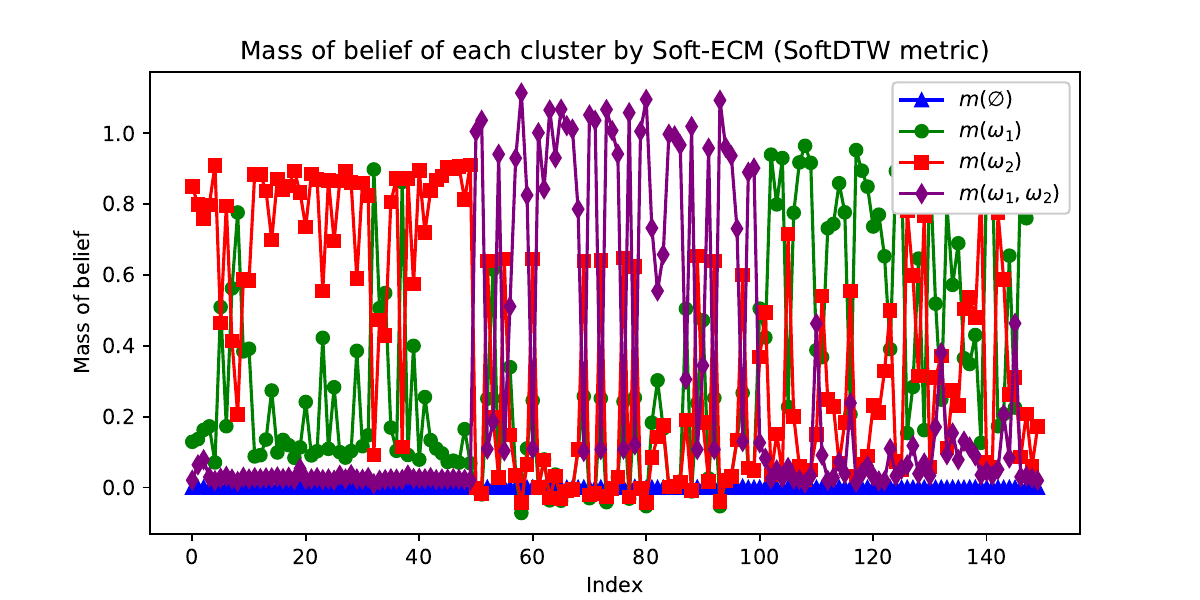}
\caption{Illustration of the three-cluster raw-data on top: \textit{Cylinder} (3 series in blue), \textit{Bell} (3 series in green) and \textit{Funnel} (3 series in red). On bottom, mass plot obtained by clustering with Soft-DTW for 2 classes (\textit{Bell} and \textit{Funnel}) and mixed examples (\textit{Bell+Funnel}).}
\label{fig:ts_expe2}
\end{figure}

Fig.~\ref{fig:ts_expe2} depicts the results obtained by \SECM clustering with $c=2$ and a Soft-DTW metric. We can see that the algorithm identifies two classes (here, the \textit{Bell} and \textit{Bell+Funnel} classes), and the third class is identified as a mixture of the first two. 
For other initializations, the algorithm identified sometimes \textit{Bell} and \textit{Funnel} as the two main clusters. These two results are acceptable for this synthetic data. Of particular interest is to show the ability of the model to identify examples as mixtures of the other two. This demonstrates that \SECM is able to identify meaningful meta-clusters of time series.

\section{Conclusion}
In this paper, we have proposed a generalization of the Evidential C-Means algorithm, named \SECM, to provide a fuzzy clustering solution adapted to complex data and not necessarily defined in a Euclidean space. Experimental results show that \SECM is able to produce comparable or even superior performance to benchmark algorithms.
Furthermore, we have shown the capability of \SECM to effectively achieve a fuzzy clustering with new types of data and compared a semi-metric, such as time series with a Dynamic Time Warping.
In addition, a major advantage for uncertain data analysis is its ability to offer a more effective representation of fuzziness. 
Although this extension introduces new hyperparameters and is based on a numerical solution scheme, the flexibility provided by \SECM in the use of a semi-metric and the improved computation of meta-cluster barycenters highlight its relevance for a wide scope of problems. 
However, the sensitivity to the parameter $\lambda$ remains a limitation of the method. 
We now plan to improve the relaxation of the problem using bi-level optimization. In addition, an extension of experiments to other complex datasets are development directions for further evaluation of the robustness and efficiency of the algorithm.

\section*{Acknowledgements}
V. Antoine and A. Soubeiga acknowledge the support received from the Agence Nationale de la Recherche of the French government through the program 16-IDEX-0001 CAP 20-25.

\newcommand{\etalchar}[1]{$^{#1}$}

\end{document}